\documentclass{ecai} 

\usepackage{latexsym}
\usepackage{amssymb}
\usepackage{amsmath}
\usepackage{amsthm}
\usepackage{booktabs}
\usepackage{enumitem}
\usepackage{graphicx}
\usepackage{color}
\usepackage{wrapfig}
\usepackage{multirow}
\usepackage{algorithm}
\usepackage{algpseudocode}
\usepackage{placeins}
\usepackage{multicol} 
\usepackage[backref]{hyperref}
\usepackage{stfloats}
\hypersetup{
hidelinks}

\newtheorem{theorem}{Theorem}
\newtheorem{lemma}[theorem]{Lemma}

\newcommand{\BibTeX}{B\kern-.05em{\sc i\kern-.025em b}\kern-.08em\TeX}

\begin{document}


\begin{frontmatter}


\paperid{123} 


\title{Origin Tracer: A Method for Detecting LoRA Fine-Tuning Origins in LLMs}


\author[A]{\fnms{Hongyu}~\snm{Liang}\orcid{0009-0008-3650-6587}\thanks{Corresponding Author. Email: sctz9029@sjtu.edu.cn}}
\author[A]{\fnms{Yuting}~\snm{Zheng}}
\author[B]{\fnms{Yihan}~\snm{Li}} 
\author[A]{\fnms{Yiran}~\snm{Zhang}} 
\address[A]{Shanghai Jiao Tong University}
\address[B]{National University of Defense Technology}


\begin{abstract}
As large language models (LLMs) continue to advance, their deployment often involves fine-tuning to enhance performance on specific downstream tasks. However, this customization is sometimes accompanied by misleading claims about the origins, raising significant concerns about transparency and trust within the open-source community. Existing model verification techniques typically assess functional, representational, and weight similarities. However, these approaches often struggle against obfuscation techniques, such as permutations and scaling transformations. To address this limitation, we propose a novel detection method Origin-Tracer that rigorously determines whether a model has been fine-tuned from a specified base model. This method includes the ability to extract the LoRA rank utilized during the fine-tuning process, providing a more robust verification framework. This framework is the first to provide a formalized approach specifically aimed at pinpointing the sources of model fine-tuning. We empirically validated our method on thirty-one diverse open-source models under conditions that simulate real-world obfuscation scenarios. We empirically analyze the effectiveness of our framework and finally, discuss its limitations. The results demonstrate the effectiveness of our approach and indicate its potential to establish new benchmarks for model verification. 
\end{abstract}

\end{frontmatter}


\section{Introduction}

Recently, as large language models (LLMs) continue to advance, increasingly powerful models are rapidly emerging, demonstrating exceptional performance across a wide range of tasks. Users frequently fine-tune these models to enhance their performance for specific applications. However, certain model providers have engaged in deceptive practices, exaggerating their technological capabilities for unjust gain. For example, the \href{https://huggingface.co/mattshumer/Reflection-Llama-3.1-70B}{Reflection-70B}, marketed by HyperWrite as the world’s leading open-source model, was in fact fine-tuned on Llama3-70B-instruct, not on Llama3.1-70B as originally claimed, as illustrated in Figure~\ref{fig:enter-label}. Such false claims have raised significant concerns regarding the potential misuse of models and the spread of misleading information~\citep{pan2023risk}.

Recent methods for model origin detection predominantly focus on functional behavior, representational similarity, weight similarity, training data properties, and program-level analysis~\citep{klabunde2023towards}. However, these approaches often lack rigorous formal definitions and standardized evaluation criteria, resulting in ambiguity and inconsistency when determining whether a given model is a fine-tuned derivative of a specific base model. Among these techniques, weight similarity is generally regarded as one of the most indicative metrics for identifying model lineage. Nevertheless, its effectiveness can be significantly undermined by obfuscation techniques such as parameter permutation and scaling transformations~\citep{zhou2023permutation,lee2018defending}. This vulnerability highlights the urgent need for more robust and principled detection frameworks that can reliably trace fine-tuning relationships even under adversarial obfuscation.

To address this challenge, our study introduces Origin-Tracer that can rigorously determine whether a model has been fine-tuned from a specified base model. Our approach is designed to accurately and precisely address the complexity of model fine-tuning for detection, marking a significant advance over existing techniques. Crucially, the method remains valid regardless of the permutations used, enabling accurate determination of the basis model for any derivative. Through this research, we aim to establish new standards for model verification in the open-source community and improve the transparency and trustworthiness of the sources of AI models.

To empirically validate the efficacy of our detection method, we conducted tests on a diverse set of thirty-one open-source models. Recognizing the presence of rotational transformations, we treated the model parameters as inherently unknowable, approaching each model as a gray box where only the inputs and outputs of each layer are accessible. This perspective ensures that our testing conditions reflect practical limitations typically encountered in real-world scenarios. Under these constraints, our results demonstrate that our algorithm robustly identifies fine-tuning across all tested models, confirming its broad applicability and effectiveness.

\begin{figure*}
    \centering
    \includegraphics[width=\linewidth]{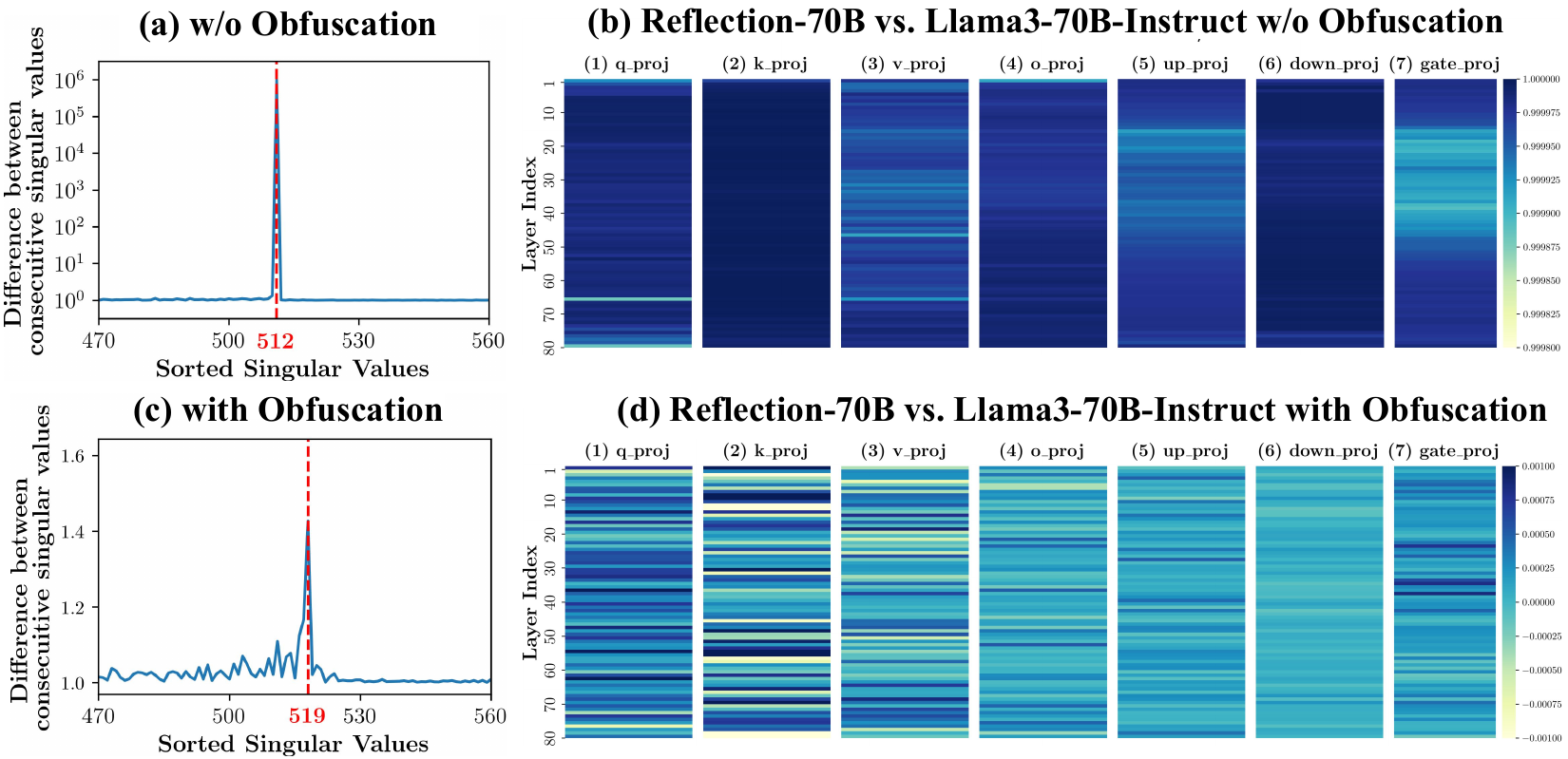}
    \caption{The detection of Reflection-70B with (w/o) obfuscation. Comparison of Origin-Tracer and Parameter Similarity Performance: Without Obfuscation (a, b) vs. With Obfuscation (c, d). Our method demonstrates resilience to obfuscation, while parameter similarity is more susceptible to its effects.}
    \vspace{0.4cm}
    \label{fig:enter-label}
\end{figure*}


\section{Related Works}

\textbf{Parameter-Efficient Fine-Tuning.} PEFT has emerged as a crucial strategy for optimizing LLMs for specific tasks while reducing resource consumption. Techniques such as Low-Rank Adaptation (LoRA)~\citep{lora,qlora}, Adapter Layers~\citep{adapter_layers}, and Prompt Tuning~\citep{prompt-tuning} achieve performance improvements by modifying only a small subset of parameters, thus capturing task-specific information while retaining the original model's foundational knowledge. However, the increasing reliance on these methods raises concerns about transparency and traceability, highlighting the need for robust verification techniques to ensure the integrity and reliability of fine-tuned models.

\textbf{Obfuscation Techniques.} To bolster model privacy and hinder unauthorized access, techniques such as permutation and scaling are employed~\cite {Elhage2021}. These methods obscure direct parameter comparisons, complicating the identification of derived models. For example, permutation rearranges parameters, scaling alters their magnitudes, and noise addition introduces random variations, all of which mask the model's characteristics. These obfuscation strategies protect intellectual property and sensitive data from unauthorized access and reverse engineering \citep{Yousefi2023}, while also preventing misuse.

\textbf{Detection Methods.} Recent researches for identifying model modifications focus on various similarities, including functional, representational, weight, training data, and procedural aspects~\citep{Klabunde2023similarity}. Functional and representational similarities compare model outputs and internal activations, respectively, but often struggle against fine-tuning variations and obfuscation techniques like permutations and noise addition~\citep{Ethayarajh2019, Wu2020similarity, Kornblith2019}. Weight similarity can effectively detect model lineage but is compromised by permutation-based obfuscation~\citep{Wang2022weight, Elhage2021}. Techniques examining training data and procedural similarities, such as influence functions, can illuminate fine-tuning practices but often require extensive datasets~\citep{Grosse2023, Shah2023}. Additionally, procedural similarity offers insights into training methods but is limited by the proprietary nature of training pipelines~\citep{Biderman2023, Zhao2023}. Overall, recent approaches highlight the challenges in detecting model modifications amid sophisticated obfuscation tactics.


\section{Preliminaries}
\label{sec:preliminaries}
\subsection{Decoder-only Transformer Architecture}

A decoder-only transformer architecture is composed of multiple decoder layers, where each layer processes an embedding matrix \( X \in \mathbb{R}^{n \times d} \) and outputs an embedding matrix of identical dimensions. Formally, let \( f: \mathbb{R}^{n \times d} \rightarrow \mathbb{R}^{n \times d} \) denote the structure of each decoder layer. Each decoder layer consists of two principal components: (1) a self-attention mechanism \( \varphi(\cdot) \), which enables the model to capture dependencies within the input sequence, and (2) a multi-layer perceptron \( \text{MLP}(\cdot) \), which facilitates nonlinear transformations and feature extraction. Given an input embedding matrix $X$, the output of a decoder layer is defined as 
$f(X) = \text{MLP}\circ\varphi(X). $
\begin{wrapfigure}{r}{0.18\textwidth}
  \vspace{-0.cm}
  \centering
  \includegraphics[width=1\linewidth]{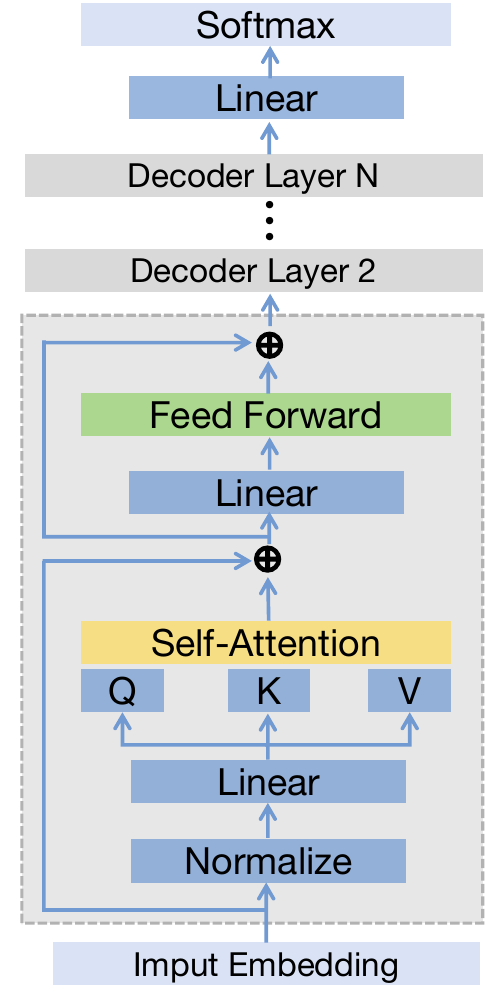}

  \caption{\small Decoder-only Architecture.}
  \label{fig:layer-architecture}
\end{wrapfigure}

The \textbf{self-attention module} \( \varphi \) comprises two main components: an input layer normalization function \( h: \mathbb{R}^{n \times d} \rightarrow \mathbb{R}^{n \times d} \), followed by a self-attention mechanism. Let \( x_i \) represent the \( i \)-th row of the embedding matrix \( X \), and let \( h_i \) denote the \( i \)-th row of the normalized output from \( h \). The relationship between \( x_i \) and \( h_i \) is given by:
$$h_i(X;\gamma) = \frac{x_i \odot \gamma}{\sqrt{\|x_i\|_2^2 + \varepsilon}},$$
where $\gamma$ is a norm weight vector. 
In the self-attention mechanism, let \( W_q \in \mathbb{R}^{d \times d_q} \), \( W_k \in \mathbb{R}^{d \times d_k} \), \( W_v \in \mathbb{R}^{d \times d_v} \), and \( W_o \in \mathbb{R}^{d_o \times d} \) denote the query, key, value, and output weight matrices, respectively ($d_q = d_k, d_v = d_o$). Additionally, let \( R_\theta \in \mathbb{R}^{d \times d} \) represent the rotary position embedding matrix (RoPE). Given the input matrix \( X \), the output of the self-attention module \( \varphi(X) \) is defined as:
$$\varphi(X) = \text{softmax}(\frac{QK^\top}{\sqrt{d_k}}) h(X;\gamma) W_v W_o + X,$$
where the key and query matrices are defined as:
$$Q = h(X;\gamma) R_\theta W_q, \quad K = h(X;\gamma) R_\theta W_k.$$
The \textbf{multi-layer perceptron}  \(\text{MLP}: \mathbb{R}^{n \times d} \rightarrow \mathbb{R}^{n \times d}\) consists of a layer normalization and a perceptron module. Here, the layer normalization \(h(\cdot; \gamma')\) mirrors the structure used in the self-attention module, though with a distinct parameter matrix \(\gamma'\). Define the weight matrices \(W_G \in \mathbb{R}^{d \times p}\), \(W_{\text{up}} \in \mathbb{R}^{d \times p}\), and \(W_{\text{down}} \in \mathbb{R}^{p \times d}\) for the gate, up, and down transformations in the perceptron, respectively. Let \(\sigma\) represent an activation function, such as GeLU or SiLU. Given an input matrix \(X\), the output of the MLP module is computed as
$$
\text{MLP}(X) = (\sigma(G) \odot U) W_{\text{down}} + X,
$$
where the gate \(G\) and up \(U\) matrices are defined as
$$
G = h(X; \gamma') W_G, \quad U = h(X; \gamma') W_{\text{up}}
$$
and $\odot$ denotes element-wise matrix product.

\subsection{Obfuscation}
Obfuscation in neural networks refers to the deliberate transformation of model parameters or architectural components to conceal their original structure while preserving functional equivalence. In practice, obfuscation techniques—such as reordering parameter matrices within attention and multilayer perceptron (MLP) modules—are employed to complicate unauthorized access, inhibit direct model comparisons, and enhance intellectual property protection. Despite altering the internal representation, these methods maintain the model’s functional output, thereby preserving performance while obscuring internal details.

Significant studies, such as those by \cite{maron2020universality} and \cite{zaheer2017deep}, have explored obfuscation's effects in maintaining consistent outputs across different configurations. These works highlight how obfuscation stabilizes model performance and hinders reverse engineering efforts.

Mathematically, given a set \( S = \{s_1, s_2, \dots, s_n\} \), obfuscation is defined through transformations that render the underlying structure opaque. When applied to a weight matrix \( W \in \mathbb{R}^{m \times n} \), transformation matrices \(\Pi \in \{0, 1\}^{n \times n} \) reorder elements, turning \( W \) into \(\Pi( W) \). These transformations complicate direct analysis without affecting the model’s output.

In Transformer architectures, the MLP and attention layers, denoted by MLP and \( \varphi \), undergo obfuscation through transformations \( \Pi_1 \) and \( \Pi_2 \), defined as follows:
\[
\text{MLP}_{\text{obf}} \circ \varphi_{\text{obf}}(X) = \text{MLP} \circ \varphi(X),
\]
where \( \text{MLP}_{\text{obf}} = \Pi_1(\text{MLP}) \) and \( \varphi_{\text{obf}} = \Pi_2(\varphi) \). This approach ensures that internal obfuscation does not affect the overall output, maintaining the model's integrity and safeguarding its internal structure.


\subsection{Problem Formulation}
This research aims to determine whether the candidate model \( M_c \) has undergone fine-tuning in its self-attention modules, excluding MLP modules, from the base model \( M_b \) using Low-Rank Adaptation (LoRA), followed by layer-level obfuscation. We consider both models as white boxes, but the obfuscations in \( M_c \) complicate comparisons with \( M_b \) due to potential parameter transformations that may obscure the structural relationships between their parameter matrices.

Let $M_c^*$ represent the ideally fine-tuned model derived from $M_b$ using Low-Rank Adaptation (LoRA) without any obfuscation. The candidate model $M_c$ is then generated from $M_c^*$ by implementing obfuscations to its layers. The challenge posed by this scenario is encapsulated by the discrepancy in ranks of the parameter differences, expressed as:
$$\text{rank}(W_c^* - W_b) = s ~,~ \text{rank}(W_c - W_b) \gg s ,$$
$$\text{MLP}_c^* = \text{MLP}_b~,~ \text{MLP}_c = \Pi(\text{MLP}_c^*).$$
Here, \( W \) represents the matrices of fine-tuned module parameters, \( W_c^* \) denotes the ideally fine-tuned matrices without obfuscations, and \( W_b \) represents the parameter matrices of the base model.

The primary challenge this research addresses is the determination of the original, unpermuted parameter matrix $W_c^*$ given the observed permuted matrix $W_c$. Our primary objective is to develop methodologies by which the structure of $W_c^*$ can be accurately inferred from $W_c$ without prior knowledge of the specific obfuscations applied.
\section{Methodology}
\label{sec:methodology}

\begin{lemma}
\label{function f}
    For a given $x \in \mathbb{R}^{n \times d}$, MLP $:\mathbb{R}^{n \times d} \rightarrow \mathbb{R}^{n \times d}$ defined as
   $$\text{MLP}(X) = (\sigma(G) \odot U) W_{\text{down}} + X,$$
where the gate \(G\) and up \(U\) matrices are defined as
$$
G = h(X; \gamma') W_G, \quad U = h(X; \gamma') W_{\text{up}}
$$
$h(X;\gamma')$ is the normalization and $\odot$ denotes element-wise matrix product.
The MLP function is injective for non-parallel vectors.
\end{lemma}
\begin{proof}
    Assume that \(\text{MLP}(X) = \text{MLP}(Y) \) and \( X_i \not\parallel Y_i \), which implies:
    \begin{equation*}
        \left(\sigma(G_X) \odot U_X\right)W_{\text{down}} + X = \left(\sigma(G_Y) \odot U_Y\right)W_{\text{down}} + Y.
    \end{equation*}
    We can easily get $h(X;\gamma')$ is bijective for non-parallel vectors, so that can simplify the equation and rearranging it, we have:
    \begin{equation*}
        \left[\sigma\left(XW_G\right) \odot \left(XW_{\text{up}}\right)\right]W_{\text{down}} - \left[\sigma\left(YW_G\right) \odot \left(YW_{\text{up}}\right)\right]W_{\text{down}} = Y - X.
    \end{equation*}
    Suppose that the token space $\mathcal{C}$ is countable, which means that $f$:$\mathcal{C}^n \rightarrow \mathcal{C}^n.$
    Let
    \begin{multline*}
        \mathcal{M}_{ij}= \{(W_G,W_{up},W_{down})|~\text{satisfy that}~\\
        [\sigma(C_iW_G)\odot(C_iW_{up})]W_{down}=[\sigma(C_jW_G)\odot(C_jW_{up})]W_{down}\},
    \end{multline*}
    where $C_i \in \mathcal{C} \quad \text{and} \quad C_j \in \mathcal{C},$
    then we can get that
    $$dim(\mathcal{M}_{ij})=3dp-nd < 3dp,$$
    which means that
    the Lebesgue measure of $\mathcal{M}_{ij}$
    $$\mu_L(\mathcal{M}_{ij}) = 0.$$
    Suppose that $$\mathcal{M}= \bigcup_{i,j}\mathcal{M}_{ij},$$
    this indicates that 
    $$\mu_L(\mathcal{M}) = 0.$$
    So we can show that 
    $$(W_G,W_{up},W_{down}) \not \in \mathcal{M}~~\text{with probability}~1.$$
    Hence, we can say function MLP is injective for non-parallel vectors.
\end{proof}
\begin{lemma}
\label{softmax}
    Let \( A \) and \( B \) be distinct matrices of the same dimension, i.e., \( A, B \in \mathbb{R}^{m \times n} \). Then, we have that
    $$\text{softmax}(A) = \text{softmax}(B)$$
    indicates that 
    $$(A - B) = 
    \begin{pmatrix}
    \alpha_1 \\
    \alpha_2\\
    \vdots \\
    \alpha_m
    \end{pmatrix} \mathbf{1}_n^\top,$$
    where \( \mathbf{1}_n \) is the column vector of ones of length \( n \) and \( \alpha_1, \ldots, \alpha_m \) are scalars. 
\end{lemma}
\begin{proof}
    Let matrices $A$  and $B$  be defined as follows:    
    For matrices \( A \) and \( B \), the condition
    $$\text{softmax}(A) = \text{softmax}(B) ,$$
    which means that
    $$
    \left(\frac{e^{a_{i1}}}{\sum_{j=1}^{n} e^{a_{ij}}}, \ldots, \frac{e^{a_{in}}}{\sum_{j=1}^{n} e^{a_{ij}}}\right)=
    \left(\frac{e^{b_{i1}}}{\sum_{j=1}^{n} e^{b_{ij}}}, \ldots, \frac{e^{b_{in}}}{\sum_{j=1}^{n} e^{b_{ij}}}\right).$$
    This indicates that
    $$a_{ij}-b_{ij}=a_{ik}-b_{ik}$$
    Hence, we have concluded that:
    $$(A - B) = \begin{pmatrix}
    \alpha_1 \\
    \alpha_2\\
    \vdots \\
    \alpha_m
    \end{pmatrix} \mathbf{1}_n^\top.$$
\end{proof}
\begin{theorem}
\label{uniquely}
Given identical inputs and outputs, the parameter matrix of a single decoder-layer's value and output modules are uniquely determined.
\end{theorem}

\begin{proof}
Recall that the output of a single decoder-layer is the concatenation of a residual MLP  and a residual self-attention $\varphi$, where
$$f(X) = \text{MLP}\circ\varphi(X),$$
$$\varphi(X) = \text{softmax}(\frac{QK^\top}{\sqrt{d_k}}) h(X;\gamma) W_v W_o + X,$$
where the key and query matrices are defined as:
$$Q = h(X;\gamma) R_\theta W_q, \quad K = h(X;\gamma) R_\theta W_k.$$
By Lemma~\ref{function f}, we have shown that the MLP is injective for non-parallel vectors. {\color{red}Assume} that the input vectors cannot be paralleled. This indicates that for any matrix $Y\in \mathbb{R}^{n\times d}$, there exists a unique $Z\in\mathbb{R}^{n\times d}$ such that $\text{MLP}(Z)=Y$. Next, we are going to show that for a given $Z\in\mathbb{R}^{n\times d}$ and a given $X\in\mathbb{R}^{n\times d}$, there exists a unique set of matrix $(\tilde{W}_Q,\tilde{W}_K, \tilde{W}_V, \tilde{W}_O)$ satisfying rank$(\tilde{W}_Q-W_Q)=s \ll d$, rank$(\tilde{W}_K-W_K)=s \ll d$, rank$(\tilde{W}_V-W_V)=s \ll d$ and rank$(\tilde{W}_O-W_O)=s \ll d$ such that 
$$\varphi(X;\tilde{W}_Q,\tilde{W}_K, \tilde{W}_V, \tilde{W}_O)=Z.$$
We prove it by contradiction. We now assume that there exists a set of matrix $$(\hat{W}_Q,\hat{W}_K, \hat{W}_V, \hat{W}_O)\neq (\tilde{W}_Q,\tilde{W}_K, \tilde{W}_V, \tilde{W}_O).$$
satisfying rank$(\hat{W}_*-W_*)=s \ll d$ such that 
$$\varphi(X;\hat{W}_Q,\hat{W}_K,\hat{W}_V,\hat{W}_O))=Z.$$
This indicates that 
\begin{align*}
    &\text{softmax}(\frac{\hat{Q}\hat{K}^\top}{\sqrt{d_K}})h(X;\gamma)\hat{W}_V\hat{W}_O 
    =\text{softmax}(\frac{\tilde{Q}\tilde{K}^\top}{\sqrt{d_K}})h(X;\gamma)\tilde{W}_V\tilde{W}_O.
\end{align*}
For simplicity of notation, we define 
$$\hat{A}(X)=\text{softmax}(\frac{\hat{Q}\hat{K}^\top}{\sqrt{d_K}})h(X;\gamma), \tilde{A}(X)=\text{softmax}(\frac{\tilde{Q}\tilde{K}^\top}{\sqrt{d_K}})h(X;\gamma).$$
We note here that $\tilde{A}(X)$ and $\hat{A}(X)$ are both $n\times n$ matrices, where $n$ denotes the number of input tokens. This further indicates that
\begin{align*}
    &\tilde{A}(X)\tilde{W}_V\tilde{W}_O  -\hat{A}(X)\hat{W}_V\hat{W}_O =\mathbf{0}_{n\times d}.
\end{align*}
Consider the case which the {\color{red}input vector $x \in \mathbb{R}^{1 \times d}$} corresponds to a single token, and assume that $$\text{rank}\left(x_1,\cdots,x_d\right) = d.$$
    This implies that
    $$\text{rank}(h(x_1;\gamma),\cdots,h(x_d;\gamma) = \text{rank}\left(x_1,\cdots,x_d\right) = d,$$
    $$\text{softmax}\frac{QK^\top}{\sqrt{d_K}} = 1.$$
    Then we have
    $$\hat{A}(x) = \tilde{A}(x) = h(x;\gamma)$$
    and
    $$\left(h(x_1;\gamma),\cdots,h(x_d;\gamma)\right)(\tilde{W}_V\tilde{W}_O - \hat{W}_V\hat{W}_O) = \mathbf{0}_{d}.$$
    This shows that  $$\tilde{W}_V\tilde{W}_O = \hat{W}_V\hat{W}_O$$
\end{proof}

\begin{algorithm*}[b]
\caption{Origin Tracer}
\label{Algorithm}
\begin{algorithmic}[1]
    \State Initialize $n$ as half of the hidden state size $h$
    \State Load Tokenizer and filter token set $T$ using NLTK, such that $T$ contains $h$ words that form one-dimensional tensors after tokenization and embedding.
    \For{$i = 1$ to $h$}
        \State Select a token from $T$ and input it into the base model.
        \State Extract input and output from each layer.
        \State Send input to the corresponding layer of candidate model $M_c$ to obtain the output.
        \State Reconstruct intermediate states using $M_c$ output and the MLP module of base model $M_b$.
    \EndFor
    \State Initialize number of cycles $t$
    \For{$i = 1$ to $t$}
        \State $List \gets \text{RandChose}(n, T)$ \Comment{Randomly select $n$ indices from $T$}
        \State $Y_i \gets \text{Compose}(Y, Y^*, List)$ \Comment{Compose matrix from the selected indices}
        \State $\gamma_1 \geq \gamma_2 \geq \cdots \geq \gamma_n \gets \text{SingularValues}(Y_i)$
        \State $rank\_List \gets \arg \max (\log \|\gamma_i\| / \|\gamma_{i+1}\|)$
    \EndFor
    \State \Return $\min(rank\_List)$
\end{algorithmic}
\end{algorithm*}

\subsection{Extraction of LoRA Rank Information}
In this section, we examine the extraction of low-rank information from the intermediate states, specifically the value and output projection matrices \( \mathbf{W}_\mathbf{V} \) and \( \mathbf{W}_\mathbf{O} \) in Transformer models. The intermediate state between the self-attention mechanism and the MLP layer is expressed as follows:
$$Y = \varphi(X) = \text{softmax}(\frac{QK^\top}{\sqrt{d_k}}) h(X;\gamma) W_v W_o + X$$
where the key and query matrices are defined as:
$$Q = h(X;\gamma) R_\theta W_q, \quad K = h(X;\gamma) R_\theta W_k,$$

$h(X;\gamma)~,~ W_V ~\text{and}~ W_O$ are the normalization, value, and output module matrices, respectively. $R_\theta$ is a Rotational Position Encoding Matrix that incorporates positional information into the token embeddings. According to Therome \ref{uniquely}, these parameter matrices are uniquely determined by their corresponding inputs-output pairs.

To facilitate the analysis and simplify the interpretation of the intermediate state, we focus on cases where the embedded tokens reduce to a \textbf{one-dimensional tensor}. Specifically, let the input tensor be $x \in \mathbb{R}^{1 \times d}$. Under this condition, the intermediate state simplifies as:
$$y = h(x;\gamma)W_V W_O+x$$

Let \( y \) and \( \tilde{y}^* \) denote the intermediate states of \( M_b \) and \( M_c^* \) for the same input tensor \( x \). This relationship can be expressed as:
\[
y - \tilde{y}^* = h(x;\gamma)\left(W_V W_O - \tilde{W}_V^* \tilde{W}_O^*\right).
\]
We can simplify this to:
$$y - \tilde{y}^* = h(x;\gamma) W_{\text{low}},$$
where \( W_{\text{low}} = W_V W_O - \tilde{W}_V^* \tilde{W}_O^* \) represents the difference reflecting the low-rank component. Assuming the input space \( x \) spans a set of linearly independent vectors that form a full-rank matrix \( X \), we have \( \text{rank}(h(X;\gamma)) = \text{rank}(X) \). Thus,
\[
Y = h(X;\gamma) W_{\text{low}}.
\]
base model and its fine-tuned counterpart. For empirical evaluation, we constructed a dataset using the Natural Language Toolkit (NLTK). Specifically, a curated vocabulary was processed through the model's tokenizer and embedding layers to generate a sufficient number of one-dimensional tensor representations suitable for subsequent analysis.

\subsection{Equivalent Intermediate Reconstruction}
if using obfuscations in models, only by the one-dimensional tensor can not extract the LoRA information, we need to reconstruction intermediate states. In this section, we explore the reconstruction of intermediate states from the \textbf{output} and the \textbf{MLP} module of \textbf{base model} and address how to resolve obfuscations involved in these processes. The relationship between single decoder-layer of $M_c$ and $M_c^*$ can be formalized as:
$$ \text{MLP}_{\text{c}}  \circ \varphi_\text{c} =  \text{MLP}_{\text{c}} ^* \circ \varphi_\text{c}^* ~\text{and}~  \text{MLP}_{\text{c}}  = \Pi_1( \text{MLP}_{\text{c}}^*), \varphi_\text{c} = \Pi_2(\varphi_\text{c}^*),$$
where $\Pi_1$ and $\Pi_2$ are unknown obfuscation operations applied to the MLP and attention parameters, respectively.
Given that $ \text{MLP}_{\text{c}} ^* = \text{MLP}_{\text{b}}$, the equation simplifies to:
$$z_c =  \text{MLP}_{\text{c}}  \circ \varphi_\text{c}(x) = \text{MLP}_{\text{b}} \circ \varphi_\text{c}^*(x),$$
which implies that:
$$\varphi_c^*(x) = \text{MLP}_{\text{c}}^{-1}(z_c).$$
Consider the equation for $z$ as follows:
$$z = [\sigma(h(y;\gamma) W_G)\odot \left(h(y;\gamma)W_{\text{up}}\right)]W_{\text{down}}+y.$$
This equation describes a non-linear transformation involving both element-wise operations and matrix multiplications, rendering the inverse mapping from the output back to the input analytically intractable.
Given the nonlinearity and complexity of this transformation, directly inferring the intermediate state $y$ from the observed output $z$ poses significant challenges. To tackle this, we adopt an iterative optimization strategy using gradient descent to approximate the original intermediate state $y$ that likely produced the observed output. The goal is to minimize the discrepancy between the MLP output and the actual observed output by adjusting $y$. The iterative update formula is expressed as:
$$y_{m+1} = y_m - \alpha \nabla \left\| f(y_m) - z_c \right\|^2$$
where $z_c$ denotes the layer output of $M_c$, $y_m$ denotes the estimated intermediate state at iteration $m$,
$\alpha$ is the learning rate, and $\nabla \left\| f(y_m) - z_c \right\|^2$ represents the gradient of the loss function with respect to $y_m$.  This loss function quantifies the squared error between the MLP output $f(y_m)$ and the target output $z_c$.
By iteratively updating $y_m$, the gradient descent algorithm aims to converge on an intermediate state $y^*$ that, when processed through the MLP, closely replicates the observed output $z_c$. This reconstruction approach facilitates the approximation of hidden intermediate states from the MLP outputs, providing a mechanism to indirectly assess and compare the internal representations across different models, such as the base model $M_b$ and the candidate model $M_c$.

\subsection{Overall process}
Given the inherent uncertainty in determining whether a model exhibits chaotic behavior, we first need to reconstruct the intermediate states of all models under evaluation. However, a challenge arises from the unknown degree of accuracy in these reconstructed intermediate states. Some intermediate states exhibit high fidelity in reconstruction, while others deviate significantly, failing to preserve the original information. To address this, we employ a method of multiple random selections, where the best result—defined as the one with the smallest rank—is selected as the final output. In order to expedite the detection process, we opt to use the hidden size as the total number of detections, rather than reconstructing the intermediate states for all one-dimensional tensors. Additionally, we leverage half of the hidden size to extract LoRA information. This choice is grounded in the observation that the rank of the LoRA matrix generally does not exceed half of the hidden size. An overview of the entire algorithm is presented in Algorithm \ref{Algorithm}.

\textbf{Rationale}
The output function of a single decoder layer is injective with probability 1. However, certain outputs that are nearly identical may correspond to intermediate states that are not sufficiently similar, posing a challenge in identifying intermediates that are adequately close to the true intermediate state. To address this, we implement a random sampling algorithm based on the hypothesis that if outputs are nearly identical, their corresponding intermediate states are likely to be similar. Assuming a sufficient number of iterations, this method is expected to reliably approximate the true rank. The probability $P$ of obtaining the true rank can be expressed as follows:

$$
P = \lim_{n \to \infty} 1 - (1 - p_s)^n = 1,
$$

where $n$ represents the number of cycles, and $p_s$ denotes the probability that all selected intermediates are sufficiently close to the true intermediate.

\section{Experiments}

\subsection{Experimental Setup}

\textbf{Models.}
We evaluate our approach on thirty-one open-source models fine-tuned with LoRA, encompassing a diverse set of architectures. Specifically, we include the \textbf{LLaMA2} series (7B, 13B, 70B), \textbf{LLaMA3} (8B, 70B), \textbf{LLaMA3.1} (8B, 70B), and \textbf{Mistral} (7B) as the base models.

\textbf{Datasets.}
For each tokenizer, we construct a dataset derived from the Natural Language Toolkit (NLTK), ensuring that each input is encoded as a single token to maintain a consistent attention score during the self-attention module. 
\subsection{Effectiveness}
We performed  LoRA rank extraction across all layers of each model to assess the effectiveness of our method. For each layer, the smallest extracted rank was systematically selected as the most representative, ensuring that the results reflect the minimal dimensionality required to capture the essential transformations within the model. A comprehensive summary of these results is provided in Table~\ref{tab:MainResult}. To further illustrate the key aspects of our analysis, Figure~\ref{peak} visualizes the peaks in the ratios of consecutive singular values.
\begin{table}[h]
    \caption{Origin-Tracer rank estimates across different types model. ± indicates variation among the top 10$\%$ layers with highest singular value ratios. "O" denotes whether the o projection was fine-tuned; if so, the expected rank is approximately twice the LoRA rank.} 
    \centering
    \vspace{0.5cm}
    \resizebox{0.49\textwidth}{!}{
    \begin{tabular}{@{}cccccc@{}}
    \toprule
    \textbf{Base} & \textbf{Size} & \textbf{Target} & \textbf{O} & \textbf{G-T} & \textbf{Origin-Tracer} \\
    \midrule
    \multirow{8}{*}{\textbf{Llama-3.1}} 
    & \multirow{5}{*}{\textbf{8B}} 
    & \href{https://huggingface.co/souththzz/llama3.1-lora}{[1]}  & × & 8 & 8{\small$ \pm 0$} \\
    && \href{https://huggingface.co/fdelduchetto/llama-3.1-8b-Instruct-math}{[2]} & × & 16 & 19{\small$ \pm 1$} \\
    && \href{https://huggingface.co/anthonysicilia/Llama-3.1-8B-FortUneDial-ImplicitForecaster}{[3]} & × & 32 & 35{\small$ \pm 0$} \\
    && \href{https://huggingface.co/faridlazuarda/valadapt-llama-3.1-8B-it-arabic}{[4]} & × & 64 & 67{\small$ \pm 1$} \\
    && \href{https://huggingface.co/safesign/dror44/llama3.18B-APL_r_128_Instruct}{[5]} & × & 128 & 130{\small$ \pm 1$} \\
    \cmidrule(r){2-6}
    & \multirow{3}{*}{\textbf{70B}} 
    & \href{https://huggingface.co/RikiyaT/Meta-Llama-3.1-70B-tac08}{[1]} & × & 16 & 16{\small$ \pm 0$}\\
    && \href{https://huggingface.co/KevinZW/llama3.1_70b_scriptV3}{[2]} & × & 16 & 16{\small$ \pm 0$}\\
    && \href{https://huggingface.co/KevinZW/llama3.1_70b_image_desV2.2}{[3]} & × & 16 & 16{\small$ \pm 0$}\\
    \midrule
    \midrule
    \multirow{8}{*}{\textbf{Llama-3}} 
    & \multirow{5}{*}{\textbf{8B}} 
    & \href{https://huggingface.co/SwastikM/Meta-Llama3-8B-Chat-Instruct-LoRA}{[1]}  & × & 8 & 8{\small$ \pm 0$} \\
    && \href{https://huggingface.co/islam23/llama3-8b-RAG_News_Finance}{[2]} & \checkmark & 16 & 32{\small$ \pm 1$} \\
    && \href{https://huggingface.co/namespace-Pt/Llama-3-8B-Instruct-80K-QLoRA}{[3]} & \checkmark & 32 & 67{\small$ \pm 0$} \\
    && \href{https://huggingface.co/Nutanix/Meta-Llama-3-8B-Instruct_KTO_lora_SupportGPT-alignment-1}{[4]} & × & 64 & 67{\small$ \pm 1$} \\
    && \href{https://huggingface.co/safesign/llama3-8b-instruct-final-less-lora-everything}{[5]} & × & 128 & 133{\small$ \pm 1$} \\
    \cmidrule(r){2-6}
    & \multirow{3}{*}{\textbf{70B}} 
    & \href{https://huggingface.co/ScaleGenAI/Llama3-lora}{[1]} & × & 8 & 8{\small$ \pm 0$}\\
    && \href{https://huggingface.co/mattshumer/Reflection-Llama-3.1-70B}{[2]} & × & 512 & 517{\small$ \pm 0$}\\
    \midrule
    \midrule
    \multirow{11}{*}{\textbf{Llama-2}} 
    & \multirow{5}{*}{\textbf{7B}} 
    & \href{https://huggingface.co/FinGPT/fingpt-mt_llama2-7b_lora}{[1]}  & × & 8 & 8{\small$ \pm 0$} \\
    && \href{https://huggingface.co/junhaos-nv/llama2-7b-ogbn-products-lora}{[2]} & × & 16 & 18{\small$ \pm 1$} \\
    && \href{https://huggingface.co/renyiyu/llama-2-7b-sft-lora}{[3]} & × & 32 & 34{\small$ \pm 0$} \\
    && \href{https://huggingface.co/dtthanh/llama-2-7b-und-lora-2.7}{[4]} & × & 64 & 66{\small$ \pm 1$} \\
    && \href{https://huggingface.co/RuterNorway/Llama-2-7b-chat-norwegian}{[5]} & × & 128 & 128{\small$ \pm 1$} \\
    \cmidrule(r){2-6}
    & \multirow{5}{*}{\textbf{13B}} 
    & \href{https://huggingface.co/FinGPT/fingpt-sentiment_llama2-13b_lora}{[1]}  & × & 8 & 8{\small$ \pm 0$} \\
    && \href{https://huggingface.co/Lajonbot/Llama-2-13b-hf-instruct-pl-lora_adapter_model}{[2]} & × & 16 & 16{\small$ \pm 1$} \\
    && \href{https://huggingface.co/RuterNorway/Llama-2-13b-chat-norwegian-LoRa}{[3]} & × & 32 & 32{\small$ \pm 0$} \\
    && \href{https://huggingface.co/Blackroot/Llama-2-13B-Storywriter-LORA}{[4]} & × & 64 & 66{\small$ \pm 1$} \\
    && \href{https://huggingface.co/zayjean/llama-2-13b_verify-bo-lora-r256-a512-d0_3K-E20}{[5]} & × & 128 & 128{\small$ \pm 1$} \\
    \cmidrule(r){2-6}
    & \multirow{1}{*}{\textbf{70B}} 
    & \href{https://huggingface.co/Yukang/Llama-2-70b-longlora-32k}{[1]} & \checkmark & 8 & 16{\small$ \pm 0$}\\
    \midrule
    \midrule
    \multirow{5}{*}{\textbf{Mistral}} 
    & \multirow{5}{*}{\textbf{7B}} 
    & \href{https://huggingface.co/CleverShovel/Mistral-7B-v0.1-paper-reviews-lora}{[1]} & × & 8 & 8{\small$ \pm 0$} \\
    && \href{https://huggingface.co/BlazeLlama/GeoGecko-Mistral2-7B-QLORA}{[2]} & × & 16 & 16{\small$ \pm 0$} \\
    && \href{https://huggingface.co/paragdakle/mistral-stem-lw-lora}{[3]} & \checkmark & 32 & 64{\small$ \pm 0$} \\
    && \href{https://huggingface.co/farmnetz/chef-z-mistral-7b-instruct-peft}{[4]} & \checkmark & 64 & 128{\small$ \pm 0$} \\
    && \href{https://huggingface.co/paragdakle/mistral-7b-cnndaily-lora}{[5]} & \checkmark & 128 & 256{\small$ \pm 0$} \\
   \bottomrule
    \end{tabular}
    }
    \label{tab:MainResult}
\end{table}

\subsection{Discussion}

\textbf{Impacts of Inspected Layers.}
Our findings reveal substantial variation in the effectiveness of the intermediate state reconstruction algorithm across different layers of the model. In particular, the reconstruction performance in the middle layers significantly outperforms that of the initial and final layers, as illustrated in Figure~\ref{fig:layer rank}. This suggests that layer-specific characteristics play a crucial role in the success of model reconstruction methods.

\textbf{Sensitivity.}
As demonstrated in the previous section, Origin-Tracer exhibits notably higher accuracy in estimating the LoRA (Low-Rank Adaptation) ranks within intermediate layers, compared to the front and rear layers. Experimental results show that rank estimations from these layers are more consistent with the model’s true low-rank structure, underscoring their pivotal role. To further investigate this phenomenon, we quantify the 2-norm of layer outputs, as depicted in Figure~\ref{fig:layeroutputs}. The results indicate that intermediate layers produce significantly higher 2-norm values, which may reflect their enhanced capacity for capturing and transmitting complex information. This intrinsic property could explain their superior performance in low-rank approximation tasks.

\begin{figure}[h]
  \centering
  \includegraphics[width=\linewidth]{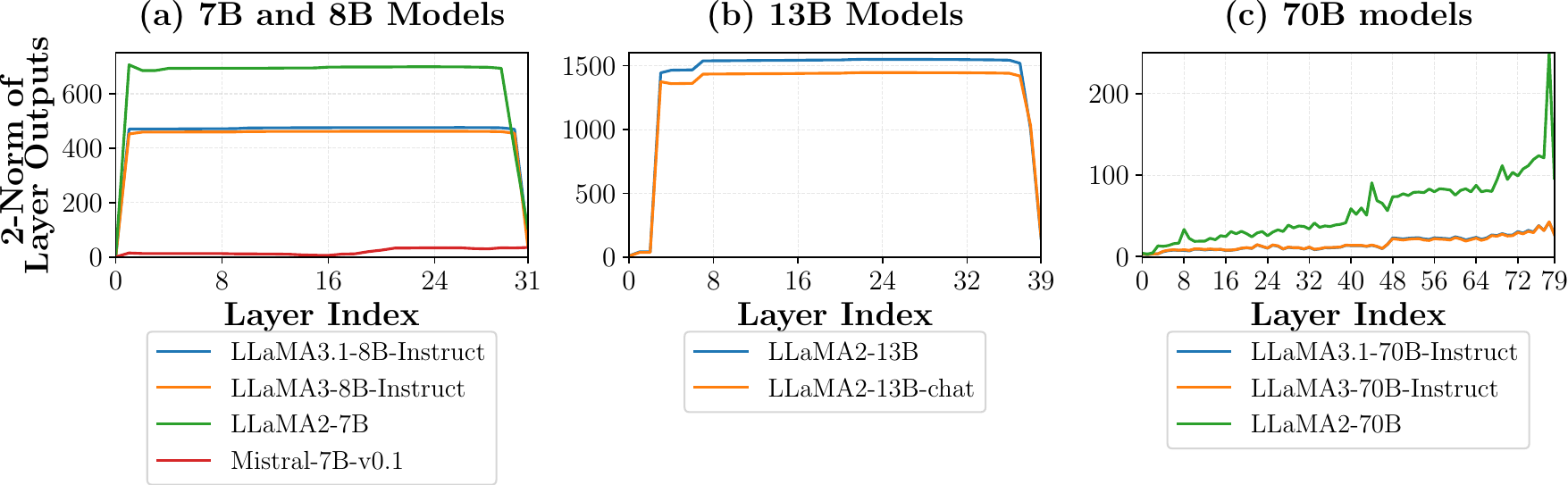}

  \caption{ Norm of Layer Outputs Across Model Architectures. This figure presents the L2 norms of outputs across layers in models of varying sizes. }
  \label{fig:layeroutputs}
\end{figure}

\begin{figure*}[h]
  \centering
  \includegraphics[width=\linewidth]{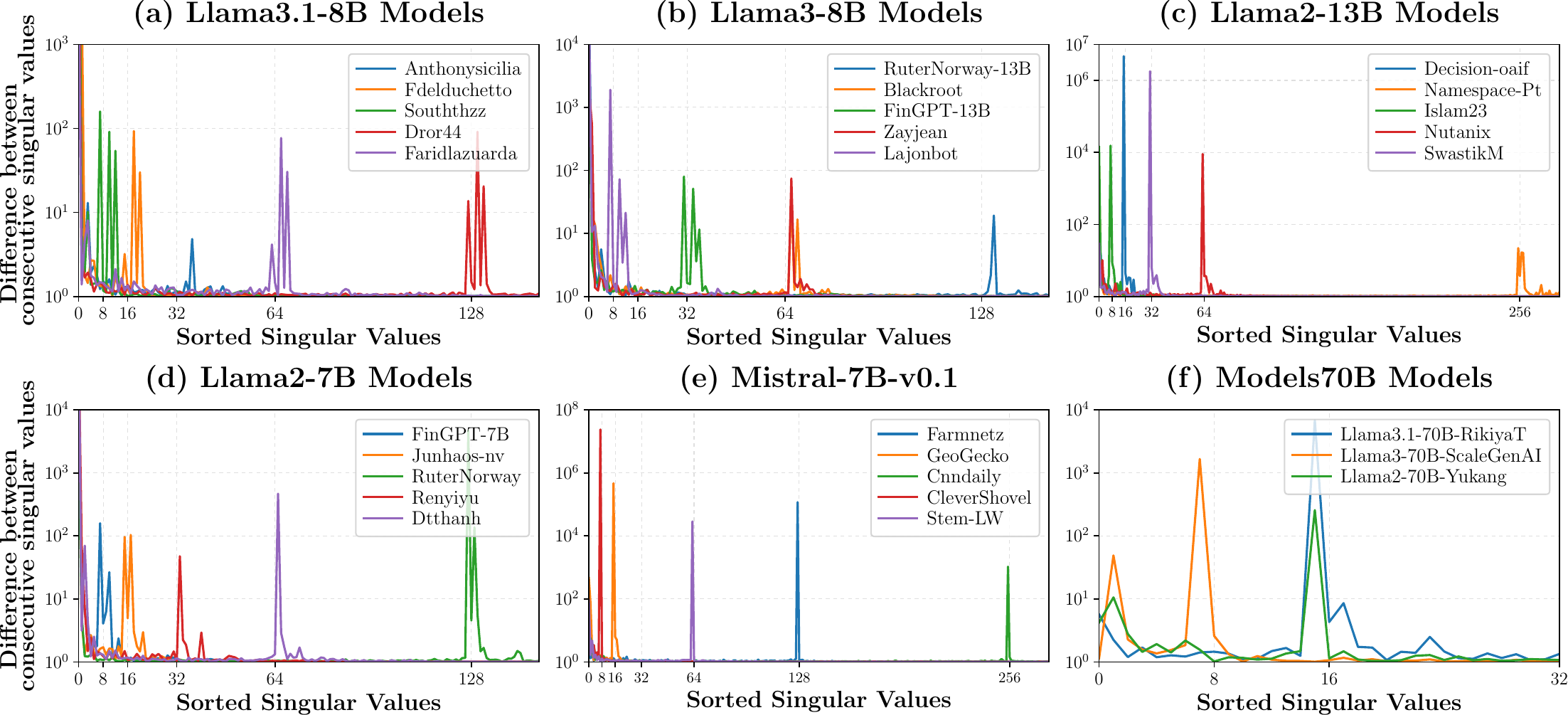}
  \caption{Origin-Tracer determines the LoRA rank by pinpointing a sharp decline in singular values, which manifests as a peak in the disparity between consecutive singular values.  In the model, this peak occurs at a position adjacent to the rank. Subfigures (a)–(f) cover LLaMA3.1-8B, LLaMA3-8B, LLaMA2-13B, LLaMA2-7B, Mistral-7B-v0.1, and 70B-scale models.}
  \label{peak}
\end{figure*}

\begin{figure*}[h]
\vspace{0.5cm}
  \centering
  \includegraphics[width=\linewidth]{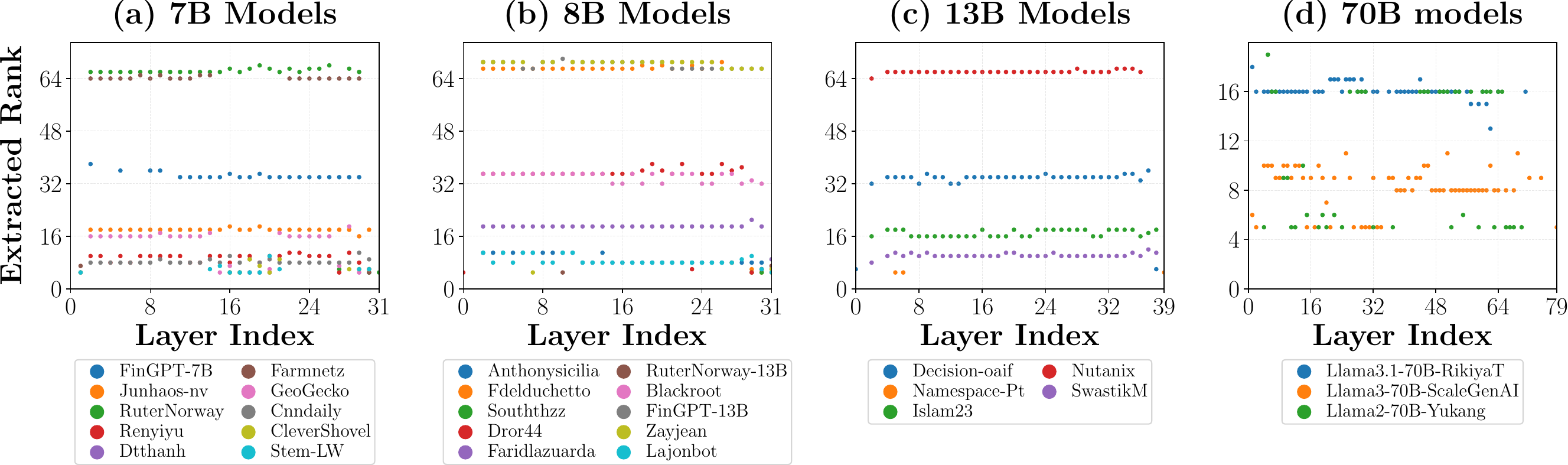}
  \caption{Layer-wise extracted ranks across different model scales.
This figure presents the extracted LoRA ranks for each transformer layer across various models. Subfigures (a)–(d) correspond to 7B, 8B, 13B, and 70B model families, respectively. Middle layers consistently exhibit higher extracted ranks, indicating more expressive transformations and suggesting their greater importance in model reconstruction.}
  \label{fig:layer rank}
  \vspace{0.5cm}
\end{figure*}

\textbf{Inspect Strategy}
We employ the origin-tracer across all layers of the model and select the top 10$\%$ of layers characterized by the highest ratios of consecutive singular values to determine their rank.  This selection criterion is based on the principle that larger ratios of consecutive singular values indicate a more distinct signal, thereby aligning more closely with the true state of the layers.

\subsection{Limitations and Future Work}
While the Origin-Tracer is effective in detecting fine-tuning origins across diverse base models, our approach has limitations that suggest several avenues for future work. 

\textbf{1. Applicability to Multi-Layer Perceptron (MLP) Changes.}  
Our method is currently not applicable to models with changes in the multi-layer perceptron (MLP) architecture. This limitation may restrict its adaptability to a wider range of neural network designs, potentially hindering its utility in various applications. Future work should explore strategies to extend the applicability of the Origin-Tracer to include models where MLP modifications are necessary.

\textbf{2. Constraints of Low-Rank Modifications.}  
Furthermore, our approach is limited to scenarios involving low-rank modifications of parameter matrices. This restriction could impact the generalization of the method to more complex model adjustments that may not conform to low-rank criteria. Subsequent research could investigate alternative strategies that accommodate a broader spectrum of parameter modifications, enhancing the flexibility of the Origin-Tracer.

\textbf{3. Parameter Modification Requirements in Attention Modules.}  
Lastly, our methodology mandates parameter modifications specifically within the 'V' (values) and 'O' (outputs) components of the attention mechanism. This requirement may limit the applicability of our approach in models where changes to other components are necessary for optimal performance. Future studies could focus on integrating modifications beyond these specified modules, thereby increasing the robustness and versatility of the Origin-Tracer.





\bibliography{ecai-template/ecai_conference}

\end{document}